\documentclass{article}

\usepackage{times}
\usepackage{xcolor}
\usepackage{soul}
\usepackage[utf8]{inputenc}
\usepackage[small]{caption}
\usepackage{floatrow}
\usepackage{subcaption}
\usepackage{hyperref}
\usepackage{stmaryrd}
\usepackage{amsmath}%
\usepackage{amssymb}%
\usepackage{amsthm}
\usepackage{amsfonts}
\usepackage{graphics}
\usepackage{gastex}
\usepackage{multirow}
\usepackage{cases}
\usepackage{algpseudocode}
\usepackage{algorithm}
\usepackage{verbatim}
\usepackage{enumitem}
\usepackage{mathrsfs}

\newfloatcommand{capbtabbox}{table}[][\FBwidth]
\newtheorem{theorem}{Theorem}

\theoremstyle{plain}

\newtheorem{definition}{Definition}

\newtheorem{lemma}{Lemma}

\newcommand{\diam}{\mathsf{diam}}
\newcommand{\dist}{\mathsf{dist}}

\newcommand{\ecc}{\mathsf{ecc}}

\newcommand{\G}{\mathfrak{G}}

\newcommand{\I}{\mathfrak{I}}
\newcommand{\IP}{\mathsf{IP}}

\newcommand{\N}{\mathbb{N}}

\newcommand{\rad}{\mathsf{rad}}

\newcommand{\RBtw}{\mathsf{RBtw}}
\newcommand{\RDeg}{\mathsf{RMax}}

\newcommand{\SDeg}{\mathsf{SMax}}
\newcommand{\SBtw}{\mathsf{SBtw}}

\usepackage{xargs}
\usepackage{xcolor}
\usepackage[colorinlistoftodos,textsize=tiny,prependcaption]{todonotes}

\usepackage{soul}
\if@todonotes@disabled

\else

\fi

\newcommandx{\unsure}[3][1=]{\texthl{#2}\todo[linecolor=red,backgroundcolor=red!25,bordercolor=red,#1]{#3}}
\newcommandx{\change}[3][1=]{\texthl{#2}\todo[linecolor=blue,backgroundcolor=blue!25,bordercolor=blue,#1]{#3}}
\newcommandx{\info}[3][1=]{\texthl{#2}\todo[linecolor=olive,backgroundcolor=olive!25,bordercolor=olive,#1]{#3}}
\newcommandx{\improve}[3][1=]{\texthl{#2}\todo[linecolor=purple,backgroundcolor=purple!25,bordercolor=purple,#1]{#3}}
\newcommandx{\thiswillnotshow}[2][1=]{\todo[disable,#1]{#2}}

\title{From the Periphery to the Center:\\ {\large Information Brokerage in an Evolving Network}}

\author{
Bo Yan$^1$,
Yiping Liu$^1$,
Jiamou Liu$^2$,
Yijin Cai$^1$,
Hongyi Su$^1$,
Hong Zheng$^1$,
\\
$^1$Beijing Institute of Technology\\
$^2$The University of Auckland\\
2120161088@bit.edu.cn$^*$,
jiamou.liu@auckland.ac.nz
}

\begin{document}

\maketitle

\begin{abstract}
Interpersonal ties are pivotal to individual efficacy, status and performance in an agent society.
This paper explores three important and interrelated themes in social network theory: the center/periphery partition of the network; network dynamics; and social integration of newcomers. We tackle the question: How would a newcomer harness information brokerage to integrate into a dynamic network going from periphery to center? We model integration as the interplay between the newcomer and the dynamics network and capture information brokerage using a process of relationship building. We analyze theoretical guarantees for the newcomer to reach the center through tactics; proving that a winning tactic always exists for certain types of network dynamics. We then propose three tactics and show their superior performance over alternative methods on four real-world datasets and four network models. In general, our tactics place the newcomer to the center by adding very few new edges on dynamic networks with $\approx 14000$ nodes.
 \end{abstract}

\section{Introduction}
An agent society (or system) is defined by patterns of dyadic links between individuals. Research on social networks has greatly advanced our understanding of how traits such as ties, modules, and  flow, impact agents' positions \cite{borgatti2011network}. Gaining a central position is seen as beneficial thanks to the relative easiness it brings to receive diverse information and exercise influence over other agents, i.e., a central position defines an {\em information broker} who accesses and integrates information through social links. This notion has wide implications on roles, status and leadership in organizations \cite{liu2016hierarchies} and has recently facilitated applications such as IoT \cite{galov2015design} and semantic web \cite{honkola2010smart}.


A predominant meso-scale feature of many complex networks is the emergence of tiers: Sitting at the center is a densely-connected cohesive {\em core}, and on the outskirts a loosely-knit {\em periphery}. 
This paper asks the question: {\em How would a newcomer harness information brokerage to integrate into a dynamic network going from the periphery to the center?} Two assumptions are made: (a) We focus on networks that have a distinguished center, e.g., a {\em core/periphery structure}; and (b) We examine the decisions and processes of relationships building. The aim is to approach the question through a formal, algorithmic lens. This demands settling two issues: (1) The first concerns representations of center and brokerage. 
We stay consistent with the framework defined in \cite{moskvina2016build} and treats information brokers as agents that give the newcomer low {\em eccentricity}, hence getting into the (Jordan) center of the network \cite{wasserman1994social}. (2) The second is about network dynamics. As the network evolves with time, the model must make sense for dynamic networks. This sets this work apart from previous works on information brokerage where only static networks are of concern and brings extra complications to the problem. Even though we phrase the problem assuming additive changes to the network (as, e.g., a citation network), our notions and techniques also apply to fully dynamic networks.

\paragraph*{\bf Contribution.} (1) We formulate integration as repeated, parallel interplays between a newcomer and the ensemble of other agents in the network it aims to join. The network evolves in the form of a sequence of snapshots in discrete time, which is determined by the initial network, the network's own evolution trace, and the newcomer's strategy for adding ties. The goal of the newcomer is to adopt a tactic that moves it from periphery to center within a finite number of steps regardless of dynamic changes of the network. (2) We study the existence of such strategies under certain reasonable conditions. In particular, when the center is  bounded -- as in many networks with a center/periphery structure -- the newcomer has a winning tactic. (3) We propose three simple tactics and compare them with two methods from \cite{moskvina2016build} which are designed for the same problem on static networks. Our tactics outperform the alternatives over four real-world dynamic networks.
We also propose four dynamic network models with varying core/periphery-ness: dynamic preferential attachment, Jackson-Rogers, rich-club and onion models, and analyze the performance of tactics over them. Our tactics bring the newcomer to the center by creating less than 10 new edges in all of the experiments performed.


\paragraph*{\bf Related work.} We focus on the individual strategies on networks with dynamic core periphery structure, which related to three categories: {\em core-periphery structure}, {\em network building problem}.
 Game-theoretical research on network formation focuses on equilibria among rational agents \cite{jackson2010social,branzei2011social}; in contrast, our paper complements this body of work by investigating optimal tactics for a single agent in a dynamic setting. Motivated from \cite{uzzi2005build}, \cite{moskvina2016build} initiates the static version of the problem under investigation, namely, building the least amount of edges to bring a newcomer to the network center.
We build on their work and study dynamic networks. 
As opposed to the static case, a desired tactic may not exist under certain forms of network dynamics.

 The definition of a graph center goes back to the work of Jordan in the 19th century and eccentricity belongs to a family of distance-based centrality indices \cite{borgatti2006graph}. Despite its simplicity, eccentricity has been useful in many places, e.g. from analyzing islands networks \cite{hage1995eccentricity} and
 the rise of the Medici family in marriage alliance network \cite{padgett1993robust} to mapping Hollywood actors/actresses \cite{harris2008combinatorics}. Although the center can be identified for any network, we specifically target at core/periphery structures \cite{borgatti2000models}. Observations of such tiered structures root in economics where the world is divided between industrial, ``core'' nations and agricultural, ``peripheral'' nations  \cite{snyder1979structural,krugman1995globalization}. Similar structures are subsequently witnessed in, e.g., social networks \cite{christley2005infection}, scientific citation \cite{doreian1985structural},
 and trading networks \cite{fricke2015core}. Agents in the core, being hubs, enjoy many benefits such as control over information and domination of resources. A crucial feature of the core, apart from its central position and density, is the stability over time \cite{csermely2013structure,rombach2017core}.

\section{Network Building in a Dynamic Network}\label{sec:problem}
A {\em social network} is an undirected graph $G\!=\!(V,E)$;
 $V$ is a set of vertices (or {\em agents}), an edge $\{u,v\}\!\in\!E$, denoted by $uv$, represents links between agents $u,v$.
The {\em distance} $\dist_G(v,u)$ is the shortest length of any path between $v$ and $u$; for $S\subseteq V$, set
$\dist_G(v,S)\!=\!\min_{u\in S} (\dist_G(v,u))$. The {\em eccentricity} of $v\in V$ is $\ecc_G(v)\!=\!\max_{u\in V} \dist_G(u,v)$. 
The {\em radius} and {\em diameter} of $G$ are, resp., 
 $\rad(G)\!=\!\min_{i\in V}\ecc(i)$ and 
 $\diam(G)\!=\!\max_{i\in V}\ecc(i)$ \cite{harris2008combinatorics}.
The {\em center} of $G$ is the set $C(G)= \{v\in V\mid \ecc(v)=\rad(G)\}$.



A {\em dynamic network} evolves in discrete-time, i.e., it consists of a (potentially infinite) list of networks $\G\!=\!G_0,G_1,G_2,\ldots$ where $G_i\!=\!(V_i,E_i)$ is the network {\em instance} at {\em timestamp} $i\!\geq\!0$.
We define {\em the set of vertices} of a dynamic network $\G$ as the set $V_\G=\cup_{i\in \N} V_i$. As $\G$ may contain infinitely many timestamps, $V_\G$ may be infinite. For any $v\in V_\G$, the set of neighbors $E_\G(v)$ is $\{u\in V_\G \mid vu\in E_i, i\in \N\}$. 
As individuals usually have an only bounded capacity to manage social links, we require that $E_\G(v)$ being finite for all  $v$. Thus the graph $(V_\G, \cup_{i\in \N} E_i)$ stays a locally-finite graph.

Two caveats exist:~Firstly, we should clarify what forms of structural changes may happen. In principle, any addition/removal of vertices/edges may occur. For the majority of this paper, however, we focus on a simpler form of dynamics where the network only makes {\em additive changes},i.e., the only allowable updates are the addition of vertices/edges. Secondly, we need a policy regarding the frequency of timestamps. One natural method is to separate consecutive instances with a fixed period, e.g., instances of a social network may be generated by day.
Another common approach is to add a timestamp only when an update occurs, e.g., a timestamp may correspond to when certain update happens between two individuals causing a new edge to form.
The exact meaning of a timestamp
 should be up to the actual application scenario.

We study the process that expands a network with new edges.
Imagine an outside agent who aims to integrate into the network and explore information within.
{\em Information brokers} refer to a set of appropriately located vertices who collectively give the newcomer
good access to the information within
the network. More formally, for $G=(V,E)$, $H=(U,F)$ ($V,U$ may or may not overlap), $G\oplus H$ denotes the network $(V\cup U, E\cup F)$. Throughout, we use $u$ to denote a {\em newcomer}. For any subset $S\subseteq V$, define
$S\otimes u$ as $(S\cup u, \{vu\mid v\in S\})$.
Thus $G\oplus (S\otimes u)$ is the resulting network obtained after integrating $u$ into $G$ via building links between $u$ and every vertex in $S$. The following definition is proposed by \cite{moskvina2016build}.
\begin{definition}\label{def:broker set}
A {\em broker set} for $G$ is  $B\subseteq V$ such that $\ecc_{G\oplus (B\otimes u)}(u)\!=\!\rad(G\oplus (B\otimes u))$.
\end{definition}
The intuition behind the definition is that by making contacts with members of a broker set, $u$ can gain maximum access to the network. As a broker set always exists for a graph $G$, the question then focuses on the size of broker sets. The
 {\em minimum broker set} problem asks for a broker set $B\subseteq V$ with the smallest cardinality and is shown to be NP-complete \cite{moskvina2016build}.

The main goal of this paper is to extend the above notions to a dynamic setting, capturing the process in which a newcomer builds ties with an evolving network to gain maximal access of information in the network.
Note that connecting to brokers embodies a dynamic process: As building relations requires effort and time, a broker set $B$ is built iteratively where edges are added for $u$ one by one, until its eccentricity becomes $\rad(G\oplus (B \otimes u))$.
Over a dynamic network $\G$, $u$ would act while the network evolves \cite{yan2017dynamic}.
From $u$'s perspective, its position relies not only on its own actions but also on the ensemble of all other agents in the network. At any timestamp, both parties make moves simultaneously, affecting the next network instance. Agent $u$'s move consists of building new edges to the current graph; the others party's move consists of updates of the form ``adding a new edge (either between two existing nodes, or a new vertex and an existing node)''.  Formally, for any network $G$, an {\em expansion} of $G$ is a network $F$ whose every connected component contains at least one node in $G$, i.e., $G\oplus F$ is a network achieved by the two types of  updates. 

\begin{definition} Fix an {\em initial network} $G_0=(V_0,E_0)$ and a {\em newcomer} $u\notin V_0$. For $k,\ell\in \N$,
an {\em integration process} (IP) is a dynamic network $\I=G_0, G_1, G_2, \ldots$ where $\forall i\geq 0$
\[
    G_{i+1} = G_i\oplus (F_i\oplus (S_i\otimes u))
\]
where $S_i$ is a set of vertices in $G_i$ that are not adjacent to $u$, and $F_i$ is an expansion of $G_i$ that does not contain $u$. 
\end{definition}
Conceptually, one can view an IP as an iterative interplay between $u$ and the network who acts as a sort of ``opponent''.
Progressing from iteration $i\geq 0$ to $i+1$
the network changes  by (i) ``attaching'' a subgraph $F_i$; this may bring more vertices and edges to $G_i$; and (ii) adding an edge between $u$ and all vertices in $S_i$. The sequence of edges $F_1,F_2,\ldots$ is called the {\em evolution trace} and the sequence of sets $S_1,S_2,\ldots$ is called the {\em newcomer strategy} of $\I$. The IP is uniquely determined by the initial network $G_0$, actions of the network (in the form of an evolution trace) and the actions of $u$ (in the form of a newcomer strategy).
The definition of a dynamic network means that an IP must satisfy a {\em locally-finiteness} {\bf (LF)} condition:

\smallskip

\noindent {\bf (LF)} Any agent (including $u$) eventually stops adding new edges, i.e., $\forall v\!\in\!V_\I \,\exists r_v\!\in\!\N\, \forall r'\!\geq\!r_v\colon$ $v$ does not appear in the network $F_{r'}\oplus (S_{r'}\otimes u)$.


\section{Information Broker in a Dynamic Network}

A question arises as to how the newcomer may choose its strategy during an IP to get into the network center. 

\begin{definition}
An IP $\I=G_0,G_1,\ldots$ is a {\em broker scheme} of $u$ if $u\in C(G_r)$ for some $r\in \N$.
\end{definition}
%
\noindent We are interested in tactics that construct a broker scheme regardless of the evolution trace.
Here, our attention is on a type of strategies where $u$ makes decisions about $S_i$ at each timestamp given only the current network instance $G_i$.
\begin{definition}
A {\em tactic} of $u$ is a function $\tau$ defined on the set of all networks such that $\tau(G)\subseteq V$ for any $G=(V,E)$ and $\forall v\in V\colon uv\in E\Rightarrow v\notin \tau(G)$. 
%
An IP $\I=G_0,G_1,\ldots$ is said to be {\em consistent} with $\tau$ if its newcomer strategy $S_1,S_2,\ldots$ satisfies that $\forall i>0\colon S_i=\tau(G_{i-1})$; we use $\IP(\tau)$ to denote the class of all IPs consistent with $\tau$. 
\end{definition}
\noindent We generalize information brokerage to the dynamic context. 

\begin{definition} Let $\mathcal{P}$ be a collection of IPs.
A {\em broker tactic} for $\mathcal{P}$ is a tactic $\beta$ such that $\mathcal{P}\cap \IP(\beta)\neq \varnothing$ and any $\I\in \mathcal{P}\cap \IP(\beta)$ is a broker scheme.
\end{definition}
\noindent The rest of the section studies the existence of broker tactics.
\begin{definition}
 Fix numbers $k>0$ and $\ell\geq 0$. An IP $\I$ is {\em $(k,\ell)$-confined} if  $F_i$ of the evolution trace $F_1,\ldots$ contains at most $\ell$ and any $|S_i|\leq k$ for all $i\geq 1$.
\end{definition}
\noindent $(k,0)$-confined IP is static where broker tactics always exist. 
%
\begin{theorem}\label{thm:k vs 1}
For any $k>1$, there exists a broker tactic of $u$ for the class of $(k,1)$-confined IPs.
\end{theorem}

\begin{proof}
Define the tactic $\tau$ of $u$ as follows: Let $\tau(G)$ be an arbitrary set of $k$ vertices in $G$ not adjacent to $u$ (or the set of all vertices not adjacent to $u$ if more than $k$ such vertices exist).
We claim that any IP $\I=G_0,G_1,\ldots$ consistent with $\tau$ is a broker scheme of $u$.
Since $k>1$, for any timestamp $s>t=(k-1)|V_0|$, the network instance $G_{s}$ contains at most one vertex not-linked to $u$. Moreover, if $u$ is linked to all other vertices in $G_s$, $\ecc_{G_s}(u)=1<\rad(G_s)$ and then $\I$ is a broker scheme. 

Now suppose that for all $s>t$, the evolution process adds a new vertex, say $x_{s}$ and an edge $x_sy_s$ where $y_s\in V_{s-1}$. Then there must be a timestamp $s>t$, such that $y_s=x_{s'}$ for some $t<s'<s$, as otherwise some vertex in $G_t$ will have an infinite degree. In the instance $G_s$, the furthest vertex from $u$ is $x_s$ with $\dist_{G_s}(u,x_s)=2$ and hence $\ecc_{G_s}(u)=2$. Clearly, $y_s$ is not connected to all vertices in $V_s$ and thus $\rad(G_s)=2$. Therefore at this timestamp, $u\in C(G_s)$, and the IP is a broker scheme.
\end{proof}


\begin{theorem}\label{thm:not exist}
No broker tactic exists for the set of $(k,\ell)$-confined IPs when $\ell\!\geq \!2$.
\end{theorem}

\begin{proof} Take a tactic $\tau$ of $u$. Inductively construct $(k,\ell)$-confined IP $G_0,G_1,\ldots$ in  $\IP(\tau)$ where $u\notin C(G_i)$ for any $i\geq 0$:
Suppose $u\!\notin\!C(G_i)$ at instance $G_i\!=\!(V_i,E_i)$. Consider the graph $H\!=\!G_i\oplus (\tau(G_i)\otimes u)$. If $u$ is not adjacent to any vertex in $H$, then clearly $u$ will not reach the center of $G_{i+1}$ regardless of $F_{i+1}$. Otherwise, take  $v\!\in\!V_i$ that is the furthest from $u$. Let $r\!=\!\dist_{H}(u,v)$.
Let $P\!=\!v,x_1,\ldots,x_\ell$ be a simple path attaches to $v$ at one end, $x_1,\ldots,x_\ell$ do not belong to $V_i$, and the edges are $\{vx_1,x_1x_2,\ldots,x_{\ell-1}x_\ell\}$. Now set $G_{i+1}=G_i\oplus (P\oplus (\tau(G_i)\otimes u))=H\oplus P$.

 Clearly, $\ecc_{G_{i+1}}(u)= r+\ell$ as the furthest vertex from $u$ is $x_\ell$. Now pick a path between $u$ and $v$, and let $w$ be the vertex along this path that is adjacent to $u$; $\dist_{G_{i+1}}(w,x_\ell)=r+\ell-1$, and for all other vertices $y$, $\dist_{G_{i+1}}(w,y)\leq \dist_{G_{i+1}}(u,y)+1\leq r+1$. As $\ell\geq 2$, $\ecc_{G_{i+1}}(w)\leq r+1< \ecc_{G_{i+1}}(u)$. This means that $u\notin C(G_{i+1})$.
\end{proof}
One can view $(k,\ell)$ as specification of an IP protocol: When $k\!>\!1\!=\!\ell$, the newcomer $u$ gains an upper hand to reach the center; If, on the other hand, $\ell\!>\!2$, $u$ may never ``catch up'' with the rest of the network. The only case left is when $k\!=\!\ell\!=\!1$, and we will focus on this case in our experiments.

To further investigate the existence of a broker tactic, we look closely at the proof of Thm.~\ref{thm:k vs 1}. The non-existence of a broker tactic is due to the fact that the center ``shifts'' as new vertices are added. This may not be the case in real-life, e.g., a core/periphery structure contains a highly stable network core meaning that the network center would be relatively stable \cite{csermely2013structure,rombach2017core}. 

\begin{definition}
Let $\G\!=\!G_0,G_1,\ldots$ be a dynamic network. We say that $\G$ has a {\em bounded center} if there exists a vertex $c\!\in\!V_\G$, and $d\!\in\!\N$ such that $\forall i\geq 0 \forall v\in C(G_i)\colon \dist_{G_i}(v,c)\leq d$.
\end{definition}
Bounded center property means that the center of the dynamic network would not expand or shift indefinitely.
The following fact easily follows from {\bf (LF)}.
\begin{lemma}\label{lem:bounded center}
$\G$ is a dynamic network with a bounded center if and only if the set $\mathcal{C}=\bigcup_{i\in \N} C(G_i)$ is a finite set.
\end{lemma}
\begin{proof} Suppose $\mathcal{C}$ is finite. Pick any $c\in \mathcal{C}$ and set $d$ as $\limsup_{i\in \N} \max_{v\in C(G_i)} \dist_{G_i}(c,v)$; $d$ must be in $\N$ as $\mathcal{C}$ is finite. Therefore $\forall i\!\geq\!0\colon C(G_i)\!\subseteq\! \{v\in V_\G\mid \dist_{G_i}(v,c)\!\leq\!d\}$. 

Conversely, suppose $\G$ has a bounded center. Then $\mathcal{C}$ is a subset of the union $D_0\cup D_1\cup \cdots \cup D_d$ where each $D_i=\{v\in V_\G\mid \exists j\in \N\colon \dist_{G_j}(v,c)= i\}$, $0\leq i\leq d$. Clearly, $D_0=\{c\}$. Suppose $D_i$ is a finite set, by {\bf (LF)}, $\bigcup_{v\in D_i} \{w\mid vw\in E_j,j\in \N\}$ is also finite. Therefore, $D_{i+1}$ is also finite, and hence $\bigcup_{i\in \N}C(G_i)$ is finite.
\end{proof}

\begin{theorem}\label{thm:bounded center}
There exists a broker tactic for the class of all $(1,\ell)$-confined IPs with a bounded center.
\end{theorem}
\begin{proof} Define tactic $\tau$ by 
$\tau(G)=\{v\}$ where $\exists w\in C(G)\colon \dist_G(v,w)\leq 1$ and $uv\notin E$ if such a vertex exists; $\tau(G)=\varnothing$ otherwise. 
To show that there is some $\I\in \IP(\tau)$ that has a bounded center, simply take $G_0\!=\!(\{x_0,v,y_0\},\{x_0v,vy_0\})$ and evolution trace $F_1,F_2,\ldots$ where the edges in each $F_i$ are $y_iy_{i-1},x_{i-1}x_i$, $i>0$. The corresponding IP $\I\in \IP(\tau)$ will set $G_i=G_{i-1}\oplus (F_i\oplus \tau(G_{i-1})$. Clearly $C(G_i)=\{u,v\}$ after timestamp 3 when all edges $uv,ux_0,uy_0$ are added, and $\tau(G_i)=\varnothing$ for all $i>3$.

Take an IP $\I\in \IP(\tau)$ with a bounded center. By Lem.~\ref{lem:bounded center}, 
 $\mathcal{C}=\bigcup_{i\in \N} C(G_i)$ is finite. By {\bf (LF)}, the set $\mathcal{C'}=\bigcup_{i\in\N}\{w\in V_\I\mid \dist_{G_i}(w,v)\leq 1, v\in \mathcal{C}\}$ is also finite.
Thus for some timestamp $t$, all edges $uv$  where $v\in \mathcal{C'}$ would have been added to $G_t$. At this timestamp, $u$ belongs to $C(G)$ and the IP is a broker scheme.
\end{proof}

\section{Cost-Effective Tactics and Evaluations} \label{sec:tactic}
As shown empirically below, real networks usually have bounded centers, but the tactic in the proof of Thm.~\ref{thm:bounded center} is unpractical as it may create arbitrarily many edges. To fix a framework where tactics are comparable, from now on, we focus on tactics $\tau$ that add a single edge in a timestamp (i.e., $|\tau(G)|\!=\!1$) until $u$ enters $C(G)$. By the {\em cost} of an IP, we mean the number of timestamps elapsed before $u$ enters the center (infinite if $u$ never enters the center). We are interested in {\em cost-effective} tactics that result in the least expected cost.

\paragraph*{\bf Uset-based tactics.} Over static networks, our problem reduces to finding a minimum broker set. The problem is shown to be NP-complete by \cite{moskvina2016build} who also gave several cost-effective tactics. At any timestamp $i$ of $\I\!=\!G_0,G_1,\ldots$, the {\em uncovered set} (Uset) $U_{i}$ is $\{v\!\in\!V_i\mid \dist_{G_i}(u,v)\!>\!\rad(G_i)\}$. Two tactics, named $\SDeg$ and $\SBtw$ resp., add an edge from $u$ to a vertex $v\!\in\!U_{u,i}$ that has maximum degree (as in $\SDeg$) or betweenness (as in $\SBtw$) centrality. Over static networks, their tactics build a {\em sub-radius dominating set} which corresponds to a broker set and is normally small, e.g., $\SDeg$ finds a broker set of size 4 on a (static) collaboration network with $>8600$ vertices.

\paragraph*{\bf Rset-based tactics.} A downside, however, lies with the Uset-based tactics over core/periphery structures: Once a link is created from $u$ to someone in the core, these tactics would forbid further links with those that are also in the core (as they are ``covered''). As a result, they result in suboptimal solutions.
We thus modify the method by allowing $u$ to link with some vertices in the Uset, as long as they are close to uncovered vertices. More formally, we define a {\em remote-center set} (Rset) at timestamp $i$ of $\I\!=\!G_0,G_1,\ldots$ as $R_i\!=\!\{v\!\in\!V\mid \dist_{G_i}(x,v)\!>\!\rad(G_i)\}$, where $x$ is a furthest vertex from $u$. We introduce $\RDeg$ and $\RBtw$ as tactics that, instead of choosing vertices from $U_i$ at timestamp $i$, links $u$ with a $v$ in the Rset $R_i$ that has the largest degree (in $\RDeg$) or maximum betweenness (in $\RBtw$) centrality. To contrast these tactics, Fig.~\ref{fig:R vs U} shows an example where $\SDeg$ gave suboptimal solutions for both static/dynamic case; $\SBtw$ gave an optimal solution for static but not for the dynamic case; and $\RDeg$/$\RBtw$ gave optimal solutions for both cases.

\begin{figure}[h]
\centering
\resizebox{!}{3cm}{\includegraphics{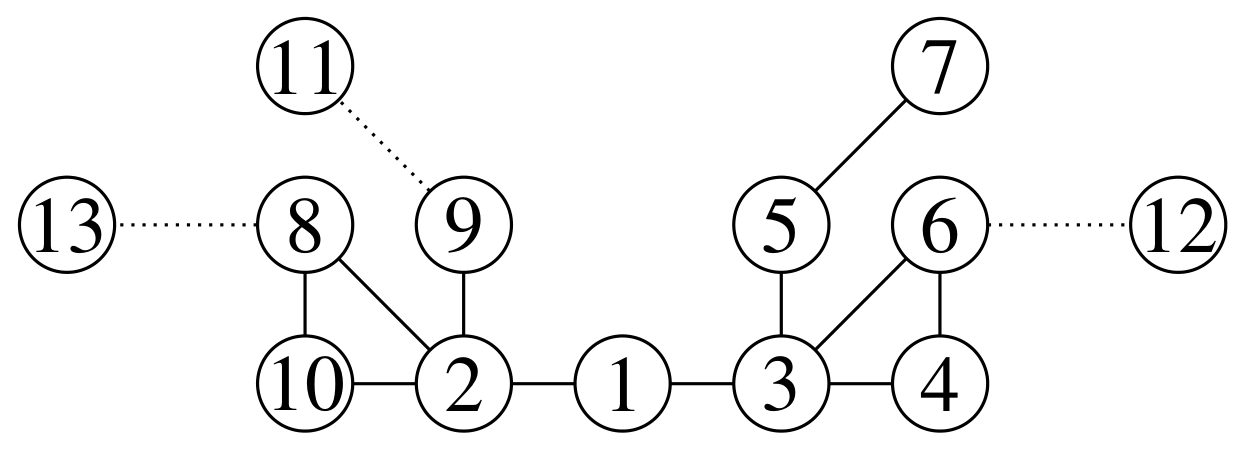}}
\caption{Contrasting $\RDeg$/$\RBtw$ with $\SDeg$/$\SBtw$. $G_0$ contains $\{1,\ldots,10\}$ and (solid) edges among them. The edges $\{9,11\}$, $\{6,12\}$, $\{8,13\}$ (dash) are added one at a time in three timestamps. Treating $G_0$ as static, $\SDeg$ and $\SBtw$ link $u$ to $\{2,6,7\}$ and $\{3,8\}$ (optimal), resp. For the dynamic network, both $\SDeg$ and $\SBtw$ create three edges in 3 timestamps ($\SDeg$ links to $2,6,7$, $\SBtw$ links to $2,5,12$), while $\RDeg$ and $\RBtw$ links to $2,3$ in 2 timestamps.}\label{fig:R vs U} 
\end{figure}

\paragraph*{\bf MUF.} Another tactic for $u$ is to link with neighbors of a center vertex $c$, thus getting into the center. To minimize cost, $c$ is chosen to have the least degree in $C(G)$. A heuristic then selects from the {\em most useful friends} (MUF) of $c$, which are defined as neighbors of $c$ that are at distance $\rad(G)-1$ from the furthest vertex from $u$.  Alg.\ref{alg:MUF} implements this tactic for one timestamp (when $u\notin C(G)$). This tactic will work in dynamic networks whose center does not change much.

%

 \begin{algorithm}[h]
 \caption{Most-Useful-Friends (MUF) tactic}\label{alg:MUF}
 \begin{algorithmic}
 \item[INPUT] A graph $G=(V,E)$, newcomer $u$
 \State $c\leftarrow \arg\min_{v\in C(G)}\deg(v)$. \Comment{center with min degree}
 \State If $u$ is isolated, return $v$ adjacent to $c$ with max degree.
 \State $x\leftarrow \arg \max_{i\in V}\dist_G(i,u)$
 \State $F\leftarrow \{v\in V\mid vc\in E,\dist_G(v,x)=\rad(G)-1\}$
 \State Return $v\!\in\!F$ not adjacent to $u$ with max degree.
 \end{algorithmic}
 \end{algorithm}

We run and evaluate the tactics on 4 real-world datasets.

\paragraph*{\bf Datasets.} The number of timestamps in these networks ranges from 50 to $\sim60000$.
\begin{description}
\item[CollegeMsg network (Msg)] is a timestamped online social network at the University of California, Irvine \cite{panzarasa2009patterns}; An edge ${jk}$ denotes a message sent between $j$ and $k$.
\item[Bitcoin OTC trust network (Bitcoin)] record anonymous Bitcoin trading on Bitcoin OTC with temporal information \cite{kumar2016edge}. 
 An edge ${jk}$ denotes a trade between $j$ and $k$. 
\item[Cit-HepPh network (Cit)]  is a high-energy physics citation network \cite{leskovec2007graph}, which collects all papers from 1992 to 1998 on arXiv; An (undirected) edge $jk$ denotes that paper $j$ cites paper $k$. 
\item[Trade network (Trade)] denotes yearly world trade partnership, 1951 -- 2009 \cite{JacksonandNei2015}; Edges represent trade partnership which is defined based on import/export between two countries.
\end{description}
All networks above, apart from Trade, has only additive changes to the network. 
Table~\ref{tab:dataset} shows multiple statistics of the last networks instance. The goodness of fit shows how well nodal degrees align with a power-law distribution, indicating a clear scale-free property. clus-coef and diam show that the networks have high clustering coefficient and low diameter, indicating small-world property. Cp-coef is a metrics for core-periphery structure; a positive value indicates a clear core/periphery structure \cite{holme2005core}.
Fig.~\ref{fig:coredist} analyzes temporal properties of the networks. It is clearly seen that, despite the continuous expansion of the network (in size), the graph center gains little in terms of diameter. Moreover, the location of the center does not shift as the maximum distance between a fixed vertex $v$ and vertices in the center stay bounded by a small distance during all timestamps.


\smallskip


\begin{table}[h]
  \centering 
  \caption{Statistics of Real-world Networks (Last Timestamp)}\label{tab:dataset}
  \begin{tabular}{|l|c|c|c|c|}
  \hline  &Trade&Msg&Bitcoin&Cit \\ \hline
  \hline $|V|$ &176&1899&5875&14083  \\
  \hline $|E|$  &1229&20296&21489&104211 \\
  \hline clust-coef &0.54&0.10&0.17&0.26  \\
  \hline max.deg &113&255&795&266 \\
  \hline diam &4&8&9&15 \\
  \hline center size &118&1&16&61  \\
  \hline timestamps &50&59835&35592&2000 \\
  \hline goodness of fit &0.74&0.89&0.86&0.91 \\
  \hline cp-coef&0.11&0.08&0.11&0.14      \\ \hline
  \end{tabular}
\end{table}

\begin{figure}
\includegraphics[width=\linewidth]{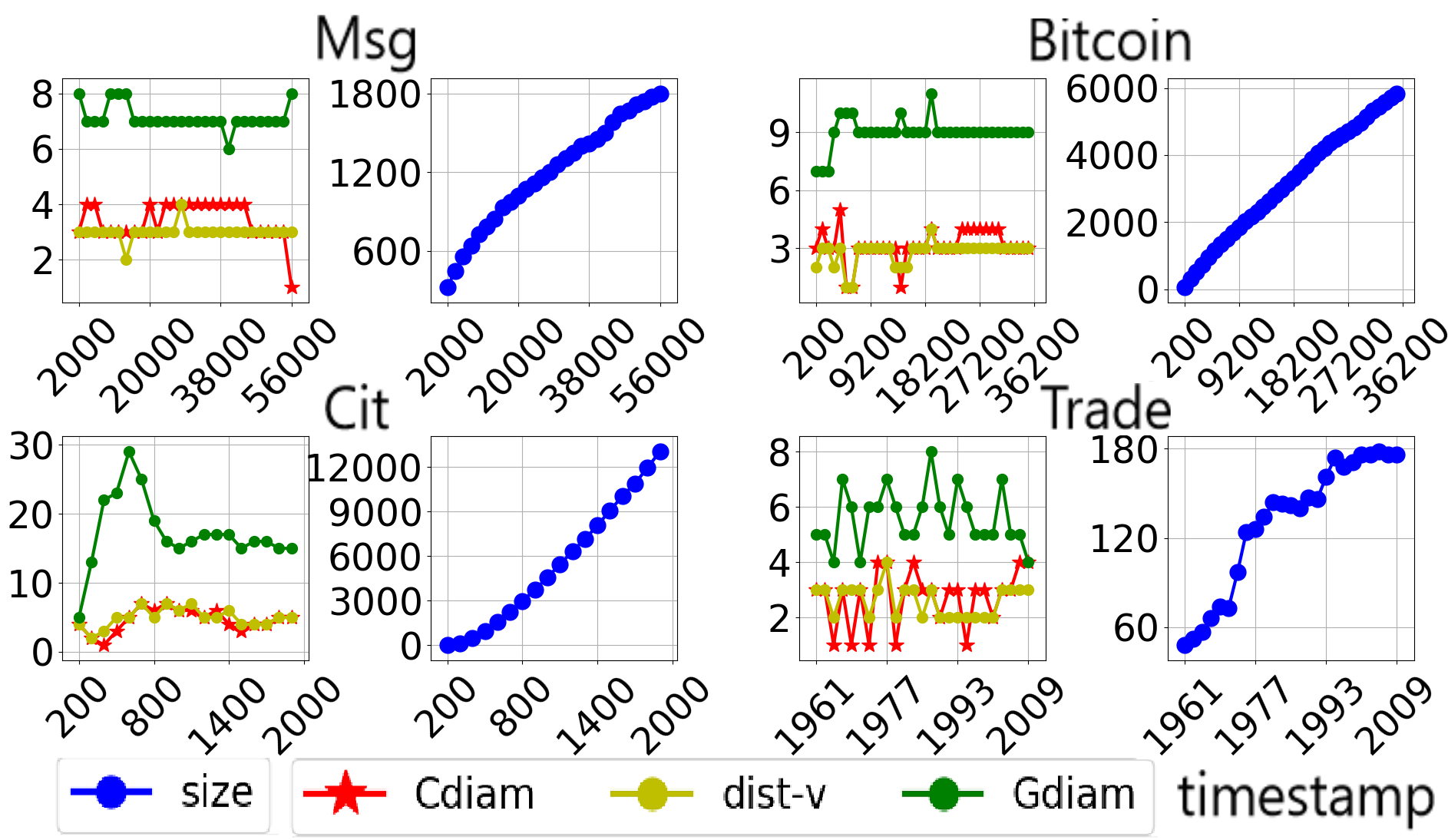}
\caption{Temporal properties of real-world networks. The horizontal axis is timestamps. \textsf{size} is the number of vertices in the network; \textsf{Gdiam} and \textsf{Cdiam} are resp. the diameter of the network and of the center. \textsf{dist}-$v$ is the maximum distance from any vertex  $C(G)$ to a fixed vertex $v$. \label{fig:coredist}}
\end{figure}


\paragraph*{\bf Experiment 1 (Cost).} The goal is to compare the tactics treating $\SDeg$/$\SBtw$ as benchmarks. For each dataset, we choose 28(Msg), 36(Bitcoin), 18(Cit), and 8(Trade) timestamps as an initial network from which IP are simulated. We also tune the interval between two consecutive timestamps where the newcomer $u$ adds an edge; See Fig.~\ref{fig:real} for results of tactics: $\RDeg,\RBtw$, MUF significantly outperform the benchmarks in all cases, obtaining costs generally below 10. They are robust in the sense that the costs vary little when starting from different initial network, while costs of $\SDeg$/$\SBtw$ dramatically increase as initial timestamp changes. To visually compare the tactics, Fig.~\ref{fig:2200} illustrates the result of $\SDeg$/$\RDeg$/MUF after running on an instance of Bitcoin with 2200 initial vertices, stopping when $u$ enters the center. $\SDeg$ apparently incurs higher cost building more edges than the other two tactics. It is also apparent that $\SDeg$ connects largely to peripheral vertices, while MUF positions $u$ well into the center.

\begin{figure}
\includegraphics[width=\linewidth]{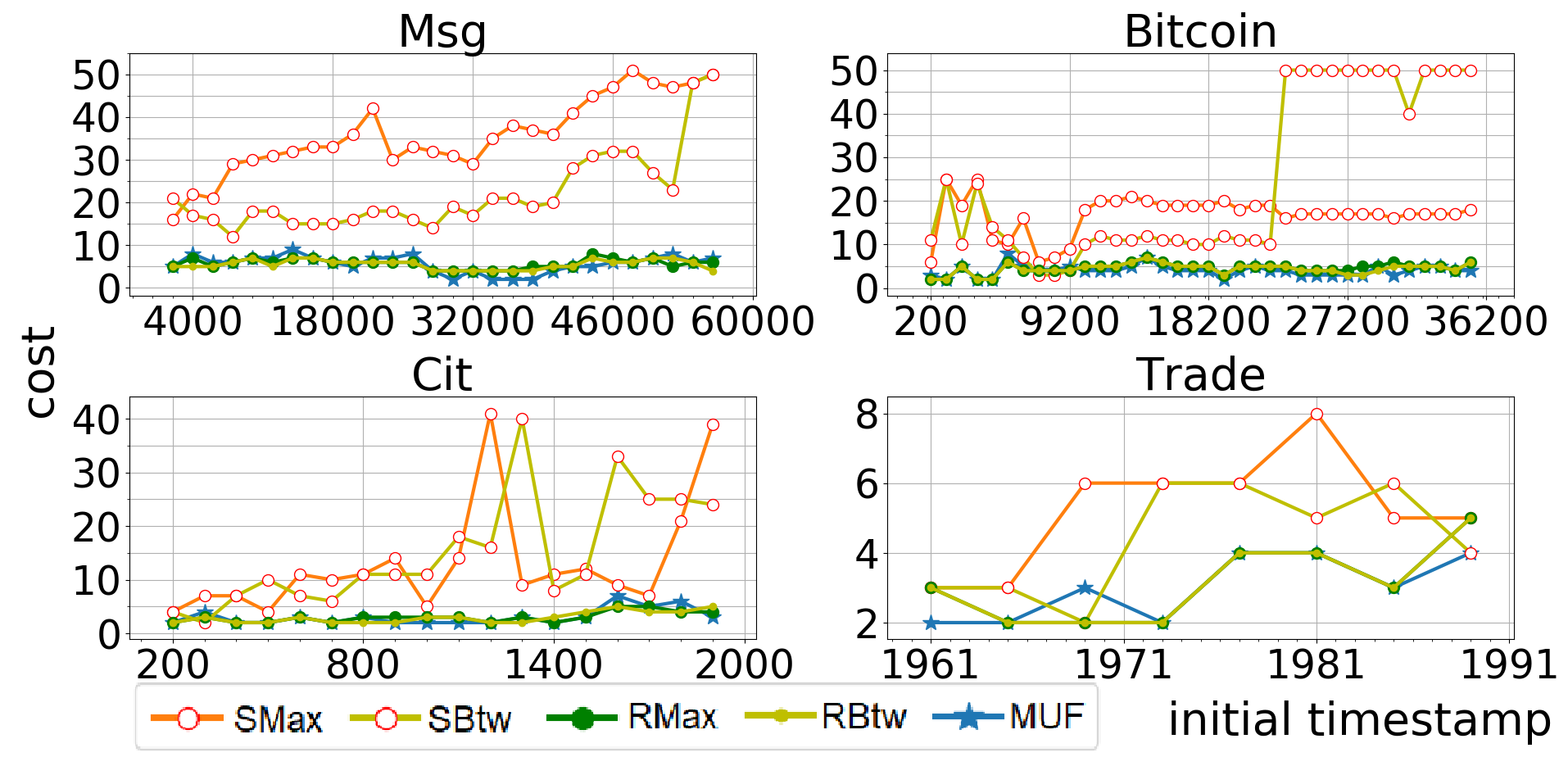}
\caption{(Top left) Msg has 28 initial timestamps and interval 50; (top right) Bitcoin has 36 initial timestamps and interval 10; (bottom left) Cit with 18 initial timestamps and interval 1; (bottom right) Trade has 8 initial timestamps and interval 1. The vertical axis indicates the cost of IP. \label{fig:real}} 
\end{figure}



\begin{figure}
  \centering
  \includegraphics[width=\textwidth]{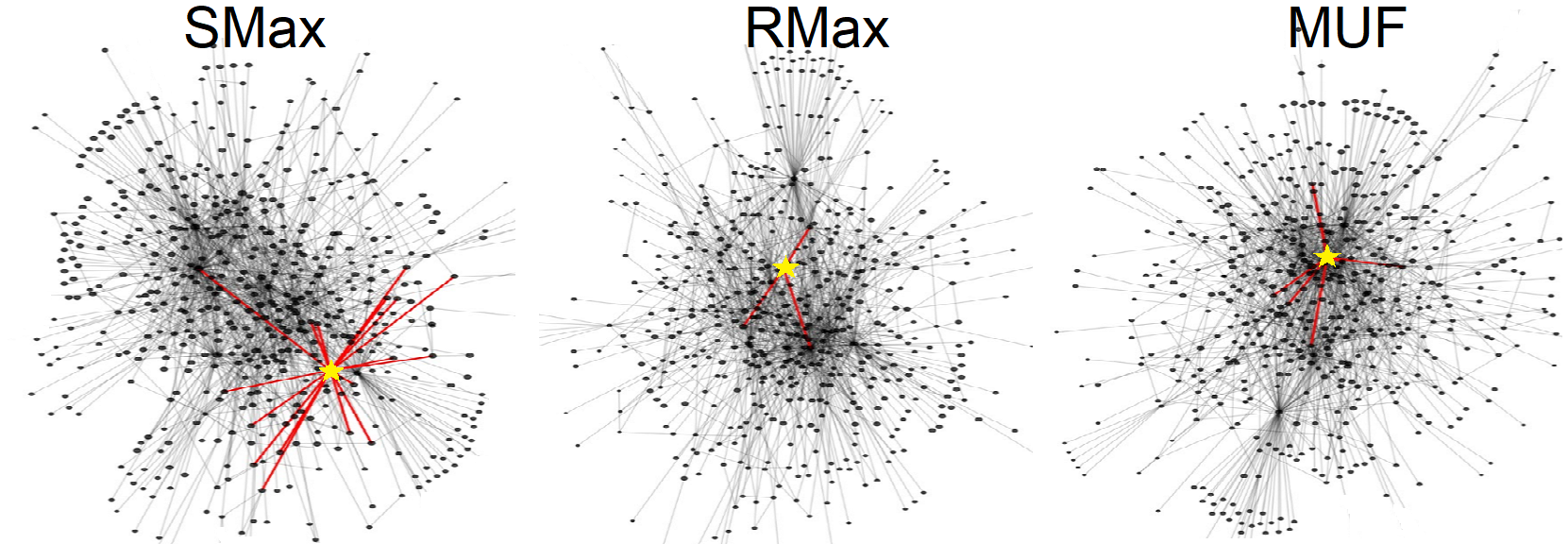}
  \caption{Result of $\SDeg$, $\RDeg$ and MUF over Bitcoin starting from 2200 vertices until $u$ (star) enters the center; red lines are edges built by the tactics.\label{fig:2200}}
\end{figure}

\section{Dynamic Center/Periphery Models}
To analyze factors attributing to tactic performance, we run dynamic network models of center/periphery structures. 


\paragraph*{\bf Dynamic BA model.} This well-established dynamic model 
takes a parameter $d\in \N$ and adds a new vertex at each timestamp who randomly links with $d$ vertices by a {\em preferential attachment} mechanism. Over multiple iterations, the graph develops a scale-free property, however, it fails to achieve a highly-clustered core.




\paragraph*{\bf Dynamic JR model.} The model proposed by \cite{jackson2007meeting} simulates stochastic friendship making among an agent population. An agent may link with a friend-of-friends or a random individual. 
At each timestamp, the model randomly samples for every vertex $v$ a set $S_1(v)$ of $m$ non-adjacent vertices from the entire network, and another set $S_2(v)$ of $m$ vertices who are at distance 2 from $v$ ($S_1(v)$ and $S_2(v)$ may not be disjoint). It then builds edges between $v$ and every vertex in $S_1(v)\cup S_2(v)$ with probability $p$. As argued in \cite{jackson2007meeting}, the model meets most of the desired properties such as scale-free and small-world properties. The value $m\approx d/4p$ relies on $p$ and an expected average degree $d\in \N$ which are parameters of the model. We pick $p=\{0.25,0.5,1\}$ to resemble the fitted values on the real-world networks in \cite{jackson2007meeting}.



\paragraph*{\bf Dynamic rich-club.} {\em Rich-club} has been a ``go-to'' model of a core/periphery structure which develops a dense, central core with a sparse periphery \cite{bornholdt2001world,csermely2013structure}.
At each timestamp, the process adds a new vertex with probability $\alpha\in[0,1]$ (and links it with a random vertex) or a link between two existing vertices with probability $1-\alpha$. If the latter case, it chooses a random source $w\in V$ and links it with a target $z$ as follows: For every $k\in \N$, set $[k]\!=\!\{v\!\in\!V\mid \deg(v)\!=\!k\}$; The probability that $z\!\in\![k]$ is $\propto\! k[k]$. The probability $\alpha$, computed by $\alpha\!=\!2(N+1)\!/\!(Nd\!+\!2)$, depends on the targeted average degree $d$ and graph size $N$, which are parameters in the model.


\paragraph*{\bf Dynamic onion.}  An {\em onion} is a core/periphery structure, but unlike in a rich-club, peripheral vertices here are connected to form one or several layers surrounding the core, resembling highly resilient networks, e.g., criminal rings
\cite{csermely2013structure}. The original static onion model generates a network with a fixed a power-law degree distribution $q(k)\!\sim\!k^{-\gamma}$ (where $\gamma\!\in\!\mathbb{R}$ depends on the average degree $d$). 
We dynamize this model so that vertices are iteratively added, loosely speaking: At each timestamp, we (1) add a new vertex $v$ whose degree $\deg(v)\!=\!k$ with probability $q(k)$; 
(2) To add $v$ to $G$ while preserving the degree distribution, create a {\em pool of ``studs''} (i.e., half-edges) initially containing $k$ studs attached to $v$; (3) randomly severe $k$ existing edges into $2k$ studs which are added to $L$; (4) repeatedly ``join'' random pairs of studs $v,w$ in $L$ to form edge $vw$ with probability $p_{vw}\!=\!(1\!+\! 3|s_v\!-\!s_w|)^{-1}$, taking care to avoid self-loops and duplicates, until $L=\varnothing$ \cite{wu2011onion}. 

Table~\ref{tab:models} summarizes key statistics of the models minding that they share the same parameter -- average degree $d\in \N$. Here we set $d=6$ to resemble values in empirical data sets\footnote{80 datasets on KONECT and SNAP have average degree between 2 and 10 \url{http://konect.uni-koblenz.de/}, \url{http://snap.stanford.edu/}}, the network size 500 and the initial network being a cycle graph with length 10, as for BA model in \cite{barabasi1999emergence}. For the JR model, a column is created for each value of $p\in \{0.25,0.5,1\}$. The rich-club and onion models have exceptionally high CP coefficient showing a clear core/periphery structure. Fig.~\ref{fig:models} visually contrasts the four models clearly displaying the core in rich-club and onion, while for BA and JR the center is not clear.


\begin{table}
  \centering   \caption{Key Statistics of the Models with $d=6$ and $N=500$}\label{tab:models}
  \resizebox{0.96\linewidth}{!}{%
  \begin{tabular}{|l|c|c|c|c|c|c|c|c}
\hline  &BA&JR0.25&JR0.5&JR1&rich-club&onion \\ \hline
 \hline clustering&0.04&0.29&0.34&0.24&0.04&0.21   \\
 \hline max deg &53.4&31.6&32.9&27.18 &46.44&118  \\
 \hline center size &120.5&33.0&47.1&19.18&27.16&3   \\
 \hline diameter &5.6&8.3&7.84&6.54&9.46&7.4\\
 \hline radius &4&5&4.7&4.08&5.2&4.18\\
 \hline cp-coef&-0.09&0.04&0.02&-0.05&0.13&0.26 \\ \hline
\end{tabular}}
\end{table}

\begin{figure}
  \centering
  \includegraphics[width=0.9\textwidth]{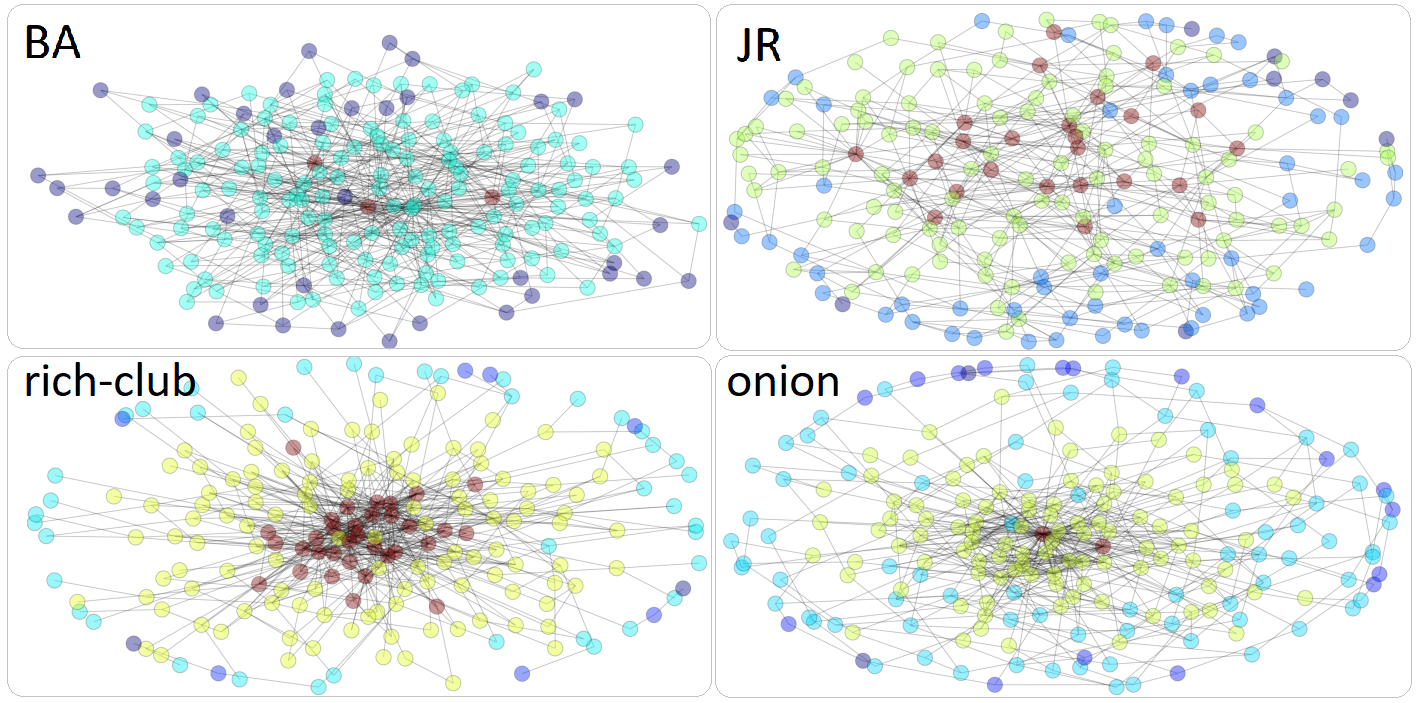}
  \caption{Illustrations of the four dynamic network models. Nodes are colored by eccentricity from lowest (blue) to highest (red).}\label{fig:models}
\end{figure}

\paragraph*{\bf Experiment 2.} We run all tactics treating $\SDeg$ \& $\SBtw$ as benchmarks over synthetic dynamic networks. IPs are simulated using the models above; the initial network is generated by the corresponding model and has size 500. There are several parameters which we may adjust. The first is the average degree $d$ which corresponds to the speed of adding edges to the network at each timestamp. 
The second is the {\em growth rate} $\ell$ of the network, which is the number of vertices that can be added in each timestamp.
Firstly, we take $d=2,\ldots,10$ and fix $\ell=1$; the costs of all tactics are plotted in Fig.~\ref{fig:ave}.
Then, we fix $d=6$ and adjust $\ell$ from 10 to 500; the costs are plotted in Fig.~\ref{fig:intervalSyn} (so that the resulting IP is $(1,\ell)$-confined). All values in Fig.~\ref{fig:ave} and Fig.~\ref{fig:intervalSyn} are averaged among 100 trials.

We make several discussions: $\bullet$ Apparent from the plots, $\RDeg$, $\RBtw$ and MUF places $u$ into the center with much less costs compared to the benchmarks; The cost of these tactics is also very stable where the cost remains below 10 for every model even when $d=14$ or $\ell=500$. Recall from Thm.~\ref{thm:not exist} that when the network expands more rapidly,  potentially no broker tactic would exist leading to an infinite cost; our experiment show that this would not happen for the four models of dynamic networks.
$\bullet$ The gap in cost between $\RDeg$/$\RBtw$/MUF and $\SDeg$/$\SBtw$ gets very wide (5 - 8 times) for models  with a high CP-coefficient (rich-club, onion). This may be due to dense ties among core members resulting in them being excluded by $\SDeg$/$\SBtw$. 
\begin{itemize}
\item[$\bullet$] Tactics have relatively similar performance over BA and JR-1 networks; This may be due to the lack of a tight-knit core in these two models.
%
%
\item[$\bullet$] A faster growth rate $\ell$ (with a fixed $d$) would not affect the costs of tactics as the tactics exploit the central vertices which are relatively stable regardless of $\ell$.
\item[$\bullet$] The vertices with high betweenness tend to locate around the center, so tactics with maximum betweenness have better performance on high CP-coefficient networks.
\end{itemize}
 \begin{figure}
  \centering
  \includegraphics[width=\linewidth]{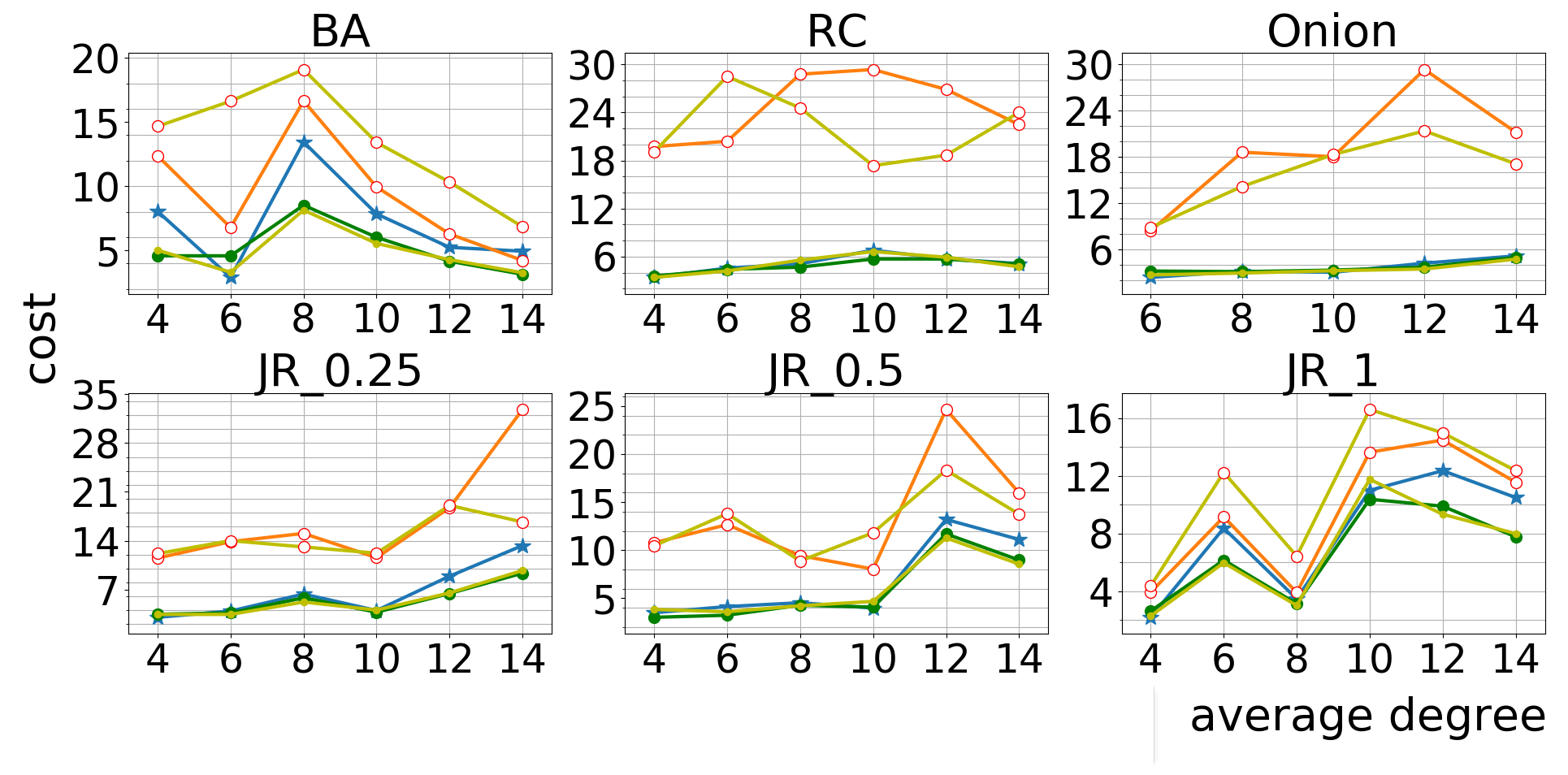}
  \caption{Costs of tactics performed on each model with varying average degree. \label{fig:ave}} 
\end{figure}




\begin{figure}
  \centering
  \includegraphics[width=\textwidth]{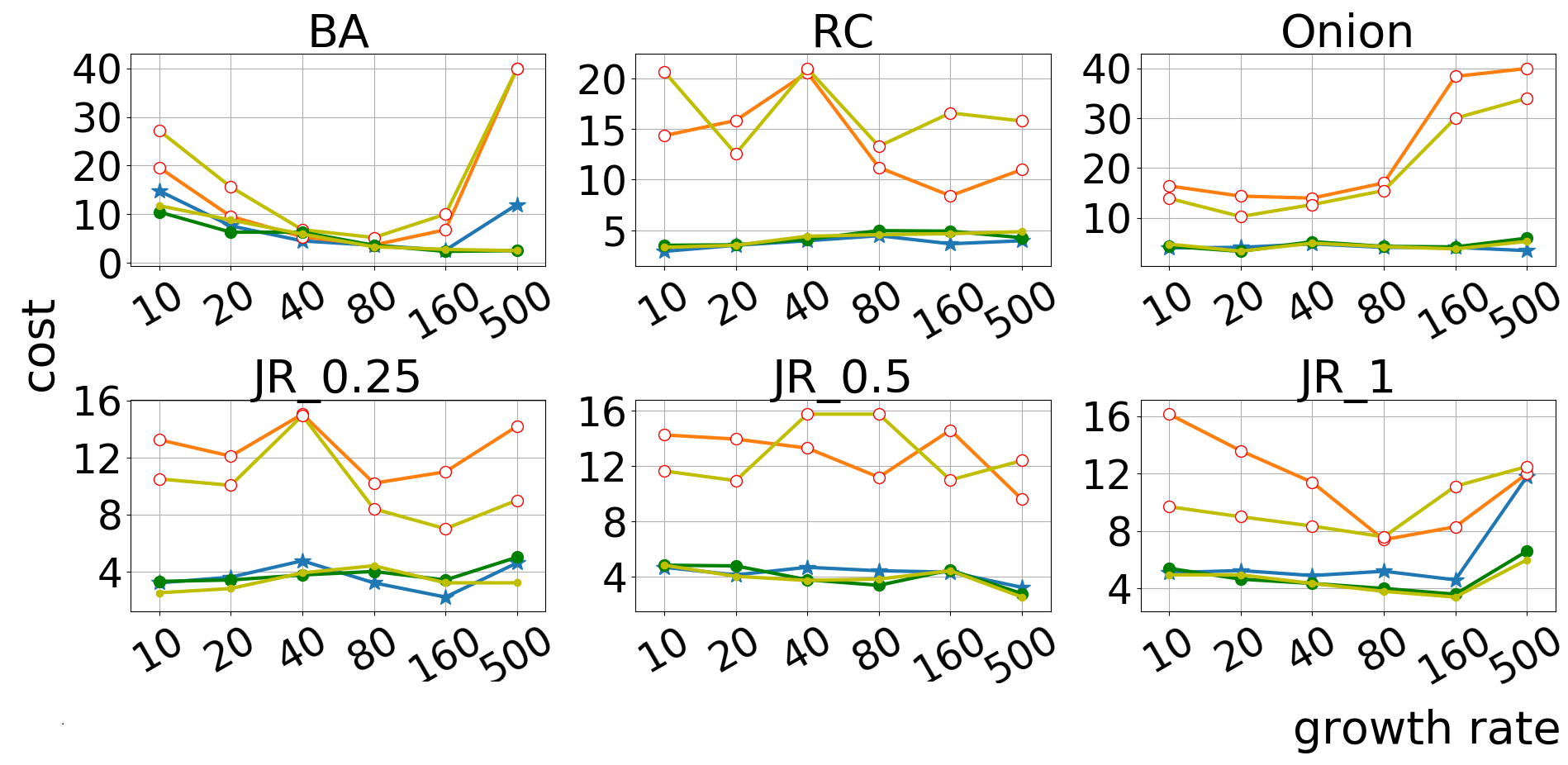}

  \caption{Costs of tactics performed on each model with varying growth rate.}\label{fig:intervalSyn}
\end{figure}



\section{Conclusion and Future Work}\label{sec:conclusion}
We develop a structural  investigation into the process where a newcomer integrates into a dynamic network through building ties. 
Our conclusions concern with conditions that warrant the existence of a broker tactic and simple cost-effective tactics over center/periphery networks. Five tactics are extensively compared on four real world datasets and four dynamic network models.

Modeling network dynamics has posed many challenges and we hope this work addresses some of them by providing a new angle and further insights. Many future work remain: (1) It is a natural question to explore dynamic models where ties are added as well as severed. (2) A distinction exists between the notions of network core and center \cite{borgatti2000models}; A future question would be to investigate tactics that place the newcomer into the core, rather than just the network center. (3) Community structure is another prevalent meso-scale property and the same question could be targeted at dynamic community structure models. (4) Moving from the tactics of a single agent to a population of agents, one may formulate and investigate game-theoretical models of network formation based on the notions of social capital.


\bibliographystyle{plain}
\bibliography{bib}

\end{document}